\documentclass{eptcs}
\providecommand{\event}{TARK 2017} 
\usepackage{breakurl}             
\usepackage{underscore}           
\usepackage{amssymb}
\usepackage{comment}
\usepackage{amsmath}
\usepackage{amsthm}
\usepackage{bbold}

\numberwithin{conjecture}{section}
\newtheorem{theorem}{Theorem}
\numberwithin{theorem}{section}
\newtheorem{lemma}{Lemma}
\numberwithin{lemma}{section}
 
\numberwithin{corollary}{section}
\newtheorem{example}{Example}
\numberwithin{example}{section}

\newcommand{\ca}{{\mathcal A}}
\newcommand{\cb}{{\mathcal B}}
\newcommand{\cc}{{\mathcal C}}

\newcommand{\cq}{{\mathcal Q}}
\newcommand{\ci}{{\mathcal I}}

\newcommand{\ct}{{\mathcal T}}

\newcommand{\bdry}{\mathsf{bdry}}
\newcommand{\cl}{\mathsf{cl}}
\newcommand{\intr}{\mathsf{int}}

\newcommand{\gs}{{\mathfrak S}}

\title{The Topology of Statistical Verifiability}

\author{Konstantin Genin
\institute{Department of Philosophy\\
Carnegie Mellon University\\
Pittsburgh, Pennsylvania}
\email{konstantin.genin@gmail.com}
\and
Kevin T. Kelly
\institute{Department of Philosophy\\
Carnegie Mellon University\\
Pittsburgh, Pennsylvania}
\email{kk3n@andrew.cmu.edu}
}
\def\titlerunning{The Topology of Statistical Verifiability}
\def\authorrunning{Konstantin Genin \& Kevin T. Kelly}
\begin{document}
\maketitle

\begin{abstract}
Topological models of empirical and formal inquiry are increasingly prevalent. They have emerged in such diverse fields as domain theory \cite{abramsky1994domain,vickers1996topology}, formal learning theory \cite{yamamoto2010topological}, epistemology and philosophy of science \cite{kelly1996logic, schulte1996topology, genin2015choice, kellygeninSL2016,baltag2015topology}, statistics \cite{dembo1994topological,ermakov2013distinguishability} and modal logic \cite{wang2013subset,bjorndahl2013subset}.  In those applications, open sets are typically interpreted as hypotheses deductively verifiable by true {\em propositional} information that rules out relevant possibilities. However, in statistical data analysis, one routinely receives random samples logically compatible with every statistical hypothesis. We bridge the gap between propositional and statistical data by solving for the {\em unique} topology on probability measures in which the open sets are exactly the {\em statistically verifiable} hypotheses. Furthermore, we extend that result to a topological characterization of learnability in the limit from statistical data.
\end{abstract}


\section{Verifiability from Propositional Information}
\label{deductiveverif}
\noindent The results in this section appear in \cite{debrecht2009}, \cite{genin2015choice}, and \cite{baltag2015topology}, but we restate them here to clarify the intended analogy between propositional and statistical verifiability. Let $W$ be a set of possible worlds, or possibilities one takes seriously, consistent with the background assumptions of inquiry. A proposition is identified with the set of worlds in which it is true, so propositions are subsets of $W$. Let $P,Q$ be arbitrary propositions. Logical operations correspond to set-theoretic operations in the usual way: $P\cap Q$ is conjunction, $P\cup Q$ is disjunction, $P^{\sf c} = W \setminus P$ is negation, and $P\subseteq Q$ is deductive entailment of $Q$ by $P$. Finally, $P$ is deductively valid iff $P=W$ and is deductively contradictory iff $P=\varnothing$.\\ 

\noindent In the propositional information setting, {\em information states} are propositions that rule out relevant possibilities. For every $w$ in $W$, let $\ci_w$ be the set of all information states true in $w$. It is assumed that ${\cal I}_w$ is non-empty (at worst, one receives the trivial information $W$).  Furthermore, it is assumed that for each $E, F$ in ${\cal I}_{w}$, there exists $G$ in ${\cal I}_{w}$ such that $G \subseteq E \cap F$.  The underlying idea is that a sufficiently diligent inquirer in $w$ eventually receives information as strong as an arbitrary information state $E$ in ${\cal I}_{w}$.  Since that is true of both $E$ and $F$, there must be true information as strong as $E \cap F$.\\

\begin{example}
\label{bairespace}
Let $W$ be the set of all infinite binary sequences. Each world $w$ determines an infinite sequence of observable outcomes. Let $w|_n$ be the initial segment of $w$ of length $n$. Let $[w|_n]$ be the set of all worlds having $w|_n$ as an initial segment. Let $\ci_w$ be the set of all $[w|_n]$ for every $n$. Think of the length of the initial segment observed as the ``stage'' of inquiry. There is exactly one such information state in $w$ at every stage, and $[w|_n]$ is entailed by $[w|_m]$ for every $m\geq n$.
\end{example}

\begin{example}
\label{realspace}
Let $W$ be the set of all real numbers. Think of the possible ``stage-$n$'' information states in $w$ as the open intervals of width $1/2^n$ that contain $w$. Then $\ci_w$ is the set of all intervals containing $w$ of width $1/2^n$, for some natural number $n$.  It follows that for every $E\in\ci_w$ there is a stage $n$ such that every stage-$n$ information state in $\ci_w$ entails $E$.
\end{example}

\noindent Let $\ci=\bigcup_w \ci_w$, the set of all possible information states. It follows from the two assumptions on ${\cal I}_{w}$ that ${\cal I}$ is a topological basis. Therefore, the closure of topological basis $\ci$ under union, denoted $\ct$, is a topological space. We assume that ${\cal I}$ is countable, since any language in which the data are recorded is at most countably infinite. The elements of $\ct$ are called {\em open sets}. The complements of open sets are called {\em closed sets}. A {\em clopen set} is both open and closed. A {\em locally closed set} is the intersection of an open and a closed set. Information state $E$ {\em verifies} proposition $H$ iff $E$ entails $H$. The {\em interior} of a proposition $H$, denoted $\intr(H)$, is the set of all worlds $w$, such that there is $E\in{\cal I}_w$ verifying $H$. Hence, $\intr(H)$ is the set of worlds in which $H$ is eventually verified by information. It is an elementary result that $H$ is open iff $H=\intr(H)$. The {\em closure} of $H$, denoted $\cl(H)$, is the set of all worlds in which $H$ is compatible with all information, i.e. $\intr(H^{\sf c})^{\sf c}$.  The boundary of $H$, denoted $\bdry(H)$, is defined as  $\cl(H)\cap \cl(H^{\sf c})$. Every information state consistent with $\bdry(H)$ is consistent with both $H$ and $H^{\sf c}$.\\

\noindent A {\em method} is a function from information states to propositions. Method $L(\cdot)$ is {\em infallible} iff its output is always true, i.e. iff $w\in L(E)$ for all $E\in\ci_w$.\footnote{This is equivalent to requiring that the method's conclusions are deductively entailed by the data, i.e. that $E\subseteq L(E)$ for all information states $E$. For this reason, infallible methods are {\em deductive} and vice-versa.} Suppose that one desires to arrive at true belief concerning some proposition $H$ without exposing oneself to the possibility of error. A {\em verifier} for $H$ is an infallible method that converges to belief in $H$ iff $H$ is true. That is, $L(\cdot)$ is a verifier for $H$ iff 
\begin{enumerate}
\item[V1.] $L(\cdot)$ is infallible and
\item[V2.] $w\in H$ iff there is $E\in\ci_w$ such that $L(F) \subseteq L(E) \subseteq H$ for all $F\in\ci_w$ entailing $E$.
\end{enumerate}
Say that $H$ is {\em verifiable} iff there exists a verifier for $H$. Say that $H$ is {\em refutable} iff its complement is verifiable, and that $H$ is {\em decidable} iff $H$ is both verifiable and refutable. For example, if you are observing a computation by an unknown program, it is verifiable that that the program will halt at some point, but it is not verifiable that it will never halt. In the setting of Example \ref{bairespace}, it is verifiable that a zero will be observed at some stage, but not that it will be observed at every stage. Verifiability is fundamentally a topological concept. 
\begin{theorem}
\label{verifiableiffopen}
\noindent Proposition $H$ is verifiable iff $H$ is open.
\end{theorem}

\noindent Theorem \ref{verifiableiffopen} implies that if $H$ is not open, then there is in general no error-avoiding method that arrives at true belief in $H$. Every method that converges to true belief in worlds in which $H$ is never verified must leap beyond the information available, and expose itself to error thereby. \\

\noindent The infallibility requirement is too strict to allow for inductive learning that draws conclusions beyond the information provided. The following success criterion is less demanding. A {\em limiting verifier} for $H$ is a method that converges to true belief in $H$ iff $H$ is true. That is, $L(\cdot)$ is a limiting verifier for $H$ iff it satisfies V2. Say that $H$ is {\em limiting verifiable} iff there exists a limiting verifier of $H$. In the setting of Example \ref{bairespace}, no method verifies the constantly-zero hypothesis $\{000\ldots \}$, but that hypothesis is verified in the limit by the method that conjectures $\{000\ldots\}$ as long as it is not refuted by information. The following is a topological characterization of the propositions that are verifiable in the limit:

\begin{theorem}
\label{limverififfsigmatwo}
Proposition $H$ is limiting verifiable iff $H$ is a countable union of locally closed sets. If $\ct$ is metrizable, then $H$ is limiting verifiable iff $H$ is a countable union of closed sets. 
\end{theorem}

\noindent Finally, an {\em empirical problem} is a countable partition $\cq$ of the worlds in $W$ into a set of {\em answers}. For $w \in W$, write $\cq_w$ for the answer true in $w$. A method is a {\em solution} to $\cq$ iff it converges, on increasing information, to the true answer in $\cq$, i.e. iff for every $w\in W$, there exists $E\in\ci_w$ such that $L(F) \subseteq \cq_w$ for all $F\in\ci_w$ entailing $E$. A problem is {\em solvable} iff it has a solution.

\begin{theorem}
\label{solvability}
Problem $\cq$ is solvable iff every answer is a countable union of locally closed sets.
\end{theorem}

\noindent  Theorems like \ref{verifiableiffopen}, \ref{limverififfsigmatwo}, and \ref{solvability} constitute an exact correspondence between topology and learnability. 
\section{Verifiability from Statistical Information}
\noindent There is a seeming gulf between propositional information and statistical samples. Propositional information literally rules out relevant possibilities. In sharp contrast, a random sample is often logically compatible with {\em every} possible probability distribution. We sidestep that fundamental difficulty by solving for the {\em unique} topology in which the open sets are precisely the {\em statistically verifiable} propositions, which provides an exact, statistical analogue of Theorem \ref{verifiableiffopen}.

\subsection{Samples and Worlds}
\noindent A {\em sample space} $\gs=(\Omega,\ct)$ is a set of possible random samples $\Omega$ equipped with a topology $\ct$ generated by a basis $\ci$. The worlds in $W$ assign probabilities to every set in $\cb$, the Borel $\sigma$-algebra generated by the topology on $\gs$. The topology on the sample space reflects what is verifiable about the sample itself. As in the purely propositional setting, it is {\em verifiable} that sample $\omega$ lands in $A$ iff $A$ is open, and  it is {\em decidable} whether sample $\omega$ falls into region $A$ iff $A$ is clopen. For example, suppose that region $A$ is the closed interval  $[1/2, \infty]$, and suppose that the sample $\omega$ happens to land right on the end-point $1/2$ of $A$. Suppose, furthermore, that given enough time and computational power, the sample $\omega$ can be specified to arbitrary, finite precision. But no finite degree of precision: $\omega \approx .50$; $\omega \approx .500$; $\omega \approx .5000$; $\ldots$ suffices to determine that $\omega$ is truly in $A$.  But the mere possibility of a sample hitting the boundary of $A$ does not matter statistically, if the chance of obtaining such a sample is zero. A Borel set $A$ for which $\mu(\bdry(A))=0$ is said to be {\em almost surely clopen (decidable) in $\mu$}.\footnote{A set that is almost surely clopen in $\mu$ is sometimes called a {\em continuity set} of $\mu$.} Borel set $A$ is almost surely clopen iff it is almost surely clopen in every $\mu$ in $W$, and a collection of Borel sets $\mathcal{S}$ is almost surely clopen iff every element of $\mathcal{S}$ is almost surely clopen.\\

\begin{example}
\label{coinflip}
Consider the outcome of a single coin flip. The set $\Omega$ of possible outcomes is $\{H, T\}$. Since every outcome is decidable, the appropriate topology on the sample space is $\ct = \{ \varnothing, \{H\}, \{T\}, \{H, T\}\}$, the discrete topology on $\Omega$. Let $W$ be the set of all probability measures assigning a bias to the coin. Since every element of $\ct$ is clopen, every element is also almost surely clopen.
\end{example}

\begin{example}
\label{reals}
Consider the outcome of a continuous measurement. Then the sample space $\Omega$ is the set of real numbers. Let the basis $\ci$ of the sample space topology be the usual interval basis on the reals. That captures the intuition that it is verifiable that the sample landed in some open interval, but it is not verifiable that it landed exactly on the boundary of an open interval. There are no nontrivial decidable (clopen) propositions in that topology. However, in typical statistical applications, $W$ contains only probability measures $\mu$ that assign zero probability to the boundary of an arbitrary open interval. Therefore, every open interval $E$ is almost surely decidable, i.e. $\mu(\bdry (E))=0$.  \\
\end{example}

\noindent {\em Product spaces} represent the outcomes of repeated sampling. Let $I$ be an index set, possibly infinite. Let $(\Omega_i, \ct_i)_{i\in I}$ be sample spaces, each with basis $\ci_i$. Define the {\em product} $(\Omega,\ct)$ of the $(\Omega_i, \ct_i)$ as follows: let $\Omega$ be the Cartesian product of the $\Omega_i$; let $\ct$ be the product topology, i.e. the topology in which the open sets are unions of Cartesian products $\times_{i} O_i$, where each $O_i$ is an element of $\ct_i$, and all but finitely many $O_i$ are equal to $\Omega_i$. When $I$ is finite, the products of basis elements in $\ci_i$ are the intended basis for $\ct$. Let $\cb$ be the $\sigma$-algebra generated by $\ct$. Let $\mu_i$ be a probability measure on $\cb_i$, the Borel $\sigma$-algebra generated by the $\ct_i$. The {\em product measure} $\mu=\times_i \mu_i$ is the unique measure on $\cb$ such that, for each $B\in \cb$ expressible as a Cartesian product of $B_i\in \cb_i$, where all but finitely many of the $B_i$ are equal to $\Omega_i$, $\mu(B)=\prod \mu_i(B_i)$. Let $\mu^{|I|}$ denote the $|I|$-fold product of $\mu$ with itself.

\subsection{Statistical Tests}
\noindent

\noindent A statistical {\em method} is a measurable function from random samples to propositions over $W$.\footnote{The $\sigma$-algebra on the range of the method is assumed to be the power set.} A {\em test} of a statistical hypothesis $H\subseteq W$ is a statistical method $\psi:\Omega\rightarrow \{W,H^{\sf c}\}$. Call $\psi^{-1}(W)$ the {\em acceptance region}, and $\psi^{-1}(H^{\sf c})$ the {\em rejection region} of the test.\footnote{The acceptance region is $\psi^{-1}(W)$, rather than $\psi^{-1}(H)$, because failing to reject $H$ licenses only the trivial inference $W$. } The {\em power} of test $\psi(\cdot)$ is the worst-case probability that it rejects truly, i.e. $\inf_{\mu \in H^{\sf c}} \mu[\psi^{-1}(H^{\sf c})]$. The {\em significance level} of a test is the worst-case probability that it rejects falsely, i.e. $\sup_{\mu \in H} \mu[\psi^{-1}(H^{\sf c})]$.\\ 

\noindent A test is {\em feasible in $\mu$} iff its acceptance region is almost surely decidable in $\mu$. Say that a test is {\em feasible} iff it is feasible in every world in $W$. More generally, say that a method is {\em feasible} iff the preimage of every element of its range is almost surely decidable in every world in $W$. Tests that are not feasible in $\mu$ are impossible to implement --- as described above, if the acceptance region is not almost surely clopen in $\mu$, then with non-zero probability, the sample lands on the boundary of the acceptance region, where one cannot decide whether to accept or reject. If one were to draw a conclusion at some finite stage, that conclusion might be reversed in light of further computation. Tests are supposed to {\em solve} inductive problems, not to generate new ones. Therefore we consider only feasible methods in the following development.

\subsection{The Weak Topology}
\noindent A sequence of measures $(\mu_n)_n$ {\em converges weakly} to $\mu$, written $\mu_n \Rightarrow \mu$, iff $\mu_n(A) \rightarrow \mu(A)$ for every $A$ almost surely clopen in $\mu$. It is immediate that $\mu_n \Rightarrow \mu$ iff for every $\mu$-feasible test $\psi(\cdot)$, $\mu_n(\psi \text{ rejects}) \rightarrow \mu(\psi \text{ rejects})$. It follows that no feasible test of $H=\{\mu\}$ achieves power strictly greater than its significance level. Furthermore, every feasible method that correctly infers $H$ with high chance in $\mu$, exposes itself to a high chance of error in ``nearby'' $\mu_n$. It is a standard fact that one can topologize $W$ in such a way that weak convergence is exactly convergence in the topology: the usual sub-basis is given by sets of the form $\{ \nu : |\mu(A) - \nu(A)| < \epsilon \}$, where $A$ is almost surely clopen in $\mu$.\footnote{Recall that a sequence $\mu_n$ converges to $\mu$ in a topology iff for every open set $E$ containing $\mu$, there is $n_0$ such that $\mu_n\in E$ for all $n\geq n_0$. If a topology is first countable, $\mu_n$ converge to $\mu$ in the topology iff $\mu$ is in the topological closure of the $\mu_n$.} That topology is called the {\em weak topology}. If $\gs$ is second countable and metrizable, then the weak topology on $W$ is also second countable and metrizable, e.g. by the Prokhorov metric \cite[Theorem 6.8]{billingsley}. When $\ci$ is countable and almost surely clopen, the weak topology is generated in a particularly natural way.\footnote{That condition is satisfied, for example, in the standard case in which the worlds in $W$ are Borel measures on $\mathbb{R}^n$, and all measures are absolutely continuous with respect to Lebesgue measure, i.e. when all measures have probability density functions, which includes normal, chi-square, exponential, Poisson, and beta distributions.  It is also satisfied for discrete distributions like the binomial, for which the topology on the sample space is the discrete (power set) topology, so every acceptance zone is clopen and, hence, feasible. Naturally, it is satisfied in the particular cases of Examples \ref{coinflip} and \ref{reals}.}

\begin{lemma}
\label{countablebasis}
Suppose that $\ci$ is a countable, almost surely clopen basis for $W$. Let $\ca$ be the algebra generated by $\ci$.  Then the collection  $\{ \mu : \mu(A) \in (a,b) \}$ for $A\in\ca$ and $a,b\in \mathbb{Q}$ is a countable sub-basis for the weak topology.\\
\end{lemma}

\noindent That sub-basis for the weak topology has two fundamental advantages over the standard sub-basis.  First, its closure under finite intersection is evidently a countable basis.  Second, it is easy to show that the sub-basis elements are statistically verifiable. The following observations are easy consequences of the Lemma. In the setting of Example \ref{coinflip}, the set of all $\{ \mu : \mu(\{H\}) \in (a,b) \}$ for $a,b\in\mathbb{Q}$, assigning open intervals of biases for the coin, forms a sub-basis for the weak topology on $W$. In fact, it forms a basis. If $\mu$ is the world in which the bias of the coin is exactly $.5$ and $\mu_n$ is the world in which the bias is exactly $.5 + 1/2^n$, then the $\mu_n$ converge to $\mu$ in the weak topology.  

\section{Statistical Verifiability}
\label{statisticalverif}
\noindent In Section 1, proposition $H$ was said to be verifiable iff there is an infallible method that converges on increasing information to $H$ iff $H$ is true. That condition implies that there is a method that achieves {\em every} bound on {\em chance} of error, and converges to $H$ iff $H$ is true.\footnote{If for every $\epsilon>0$ your chance of error is less than $\epsilon$, then your chance of error is zero: you are {\em almost surely} infallible.} In statistical settings, one cannot insist on such a high standard of infallibility. Instead, say that $H$ is {\em verifiable in chance} iff for every bound on error, there is a method that achieves it, and that converges in probability to $H$ iff $H$ is true. The reversal of quantifiers expresses the fundamental difference between statistical and propositional verifiability and, hence, between statistical and propositional information. Say that a family $\{\lambda_n\}_{n\in\mathbb{N}}$ of feasible tests of $H^{\sf c}$ is an {\em $\alpha$-verifier in chance} of $H\subseteq W$ iff for all $n\in\mathbb{N}$:

\begin{enumerate} 
\item[SV1.] $\mu^n[\lambda_n^{-1}(H)] \leq \alpha$, for all $\mu\in H^{\sf c}$ and 
\item[SV2.] $\underset{n\rightarrow \infty}{\lim} \hspace{1pt} \mu^n[\lambda_n^{-1}(H)] = 1$, for all $\mu\in H$.
\end{enumerate}

\noindent Say that $H\subseteq W$ is {\em $\alpha$-verifiable in chance} iff there is an $\alpha$-verifier in chance of $H$. Say that $H\subseteq W$ is {\em verifiable in chance} iff $H$ is $\alpha$-verifiable in chance for every $\alpha>0$.\\

\noindent The preceding definition only bounds the chance of error at each sample size. One might strengthen SV1 to the requirement that the overall chance of error be bounded, when $H$ is false. Furthermore, one might also strengthen SV2 by requiring almost sure convergence to $H$, rather than mere convergence in probability in every measure in $W$. Say that a family $\{\lambda_n\}_{n\in\mathbb{N}}$ of feasible tests of $H^{\sf c}\subseteq W$ is an {\em almost sure $\alpha$-verifier} of $H$ iff

\begin{enumerate} 
\item[SV3.] $\sum_{n=1}^\infty \mu^n[\lambda_n^{-1}(H)] \leq \alpha$ for all $\mu \in H^{\sf c}$ and 
\item[SV4.] $\mu^\infty\left[\underset{n\rightarrow \infty}{\liminf} \hspace{1pt} \lambda_n^{-1}(H)\right]=1$ for all $\mu\in H$.\\
\end{enumerate}

\noindent Say that $H\subseteq W$ is {\em almost surely $\alpha$-verifiable} iff there is an almost sure $\alpha$-verifer of $H$. Say that $H$ is {\em almost surely verifiable} iff $H$ is almost surely $\alpha$-verifiable, for every $\alpha>0$. Clearly, if $H$ is almost surely verifiable, then $H$ is verifiable in chance.\\

\noindent We now weaken the preceding two criteria of statistical verifiability to arrive at statistical notions of limiting verifiability.  Say that a family $\{\lambda_n\}_{n\in\mathbb{N}}$ of feasible methods is a {\em limiting verifier in chance} of $H\subseteq W$ iff 
\begin{enumerate}
\item  $\mu \in H \text{ iff there is } H'\subseteq H, \text{ s.t. }\underset{n\rightarrow\infty}{\lim}\hspace{1pt} \mu^n[\lambda_n^{-1}(H')]=1;$
\item $\mu \notin H \text{ iff for all } H'\subseteq H, \underset{n\rightarrow\infty}{\lim}\hspace{1pt} \mu^n[\lambda_n^{-1}(H')]=0.$
\end{enumerate}
Say that $H\subseteq W$ is {\em limiting verifiable in chance} iff there is a limiting verifier in chance of $H$.\\

\noindent As before, there is an almost sure version of that success criterion. Say that a family $\{\lambda_n\}_{n\in\mathbb{N}}$ of feasible methods is a {\em limiting almost sure verifier} of $H\subseteq W$ iff 
\begin{enumerate}
\item $\mu \in H \text{ iff there is } H'\subseteq H, \text{ s.t. } \mu^\infty[\underset{n\rightarrow\infty}{\liminf} \hspace{1pt} \lambda_n^{-1}(H')]=1;$
\item $\mu \notin H \text{ iff for all } H'\subseteq H, \text{ }  \mu^\infty[\underset{n\rightarrow\infty}{\limsup} \hspace{1pt}\lambda_n^{-1}(H')]=0.$
\end{enumerate}
Say that $H\subseteq W$ is {\em limiting a.s. verifiable} iff there is a limiting a.s. verifier of $H$.\\

\noindent Finally, there is a natural statistical analogue of solvability. Recall that an {\em empirical problem} is a countable partition $\cq$ of the worlds in $W$ into a set of {\em answers}. Say that a family $\{\lambda_n\}_{n\in\mathbb{N}}$ of feasible methods is a {\em solution in chance} to $\cq$ iff for every $\mu\in W$, $\lim_{n\rightarrow\infty} \mu^n[\lambda_n^{-1}(\cq_\mu)] = 1$. Say that $\cq$ is {\em solvable in chance} iff there exists a solution in chance to $\cq$. A family $\{\lambda_n\}_{n\in\mathbb{N}}$ of feasible methods is an {\em almost sure solution} to $\cq$ iff for every $\mu \in W$, $\mu^\infty[\liminf_{n\rightarrow\infty} \lambda_n^{-1}(\cq_\mu)]=1$. Furthermore, say that $\cq$ is {\em almost surely solvable} iff there exists an almost sure solution to $\cq$.   

\section{Results}
\label{results}
Theorem \ref{openisverifiable} states that, for sample spaces with countable, almost surely clopen bases, verifiability in chance and almost sure verifiability are equivalent to being open in the weak topology. As promised in the introduction, that fundamental result lifts the topological perspective to inferential statistics.\\

\begin{theorem}
\label{openisverifiable}
Suppose that $W$ is a set of Borel measures on $\gs=(\Omega, \ct)$, a metrizable sample space with countable, almost surely clopen basis $\ci$. Then the following are equivalent: 
\begin{enumerate}
\item $H\subseteq W$ is $\alpha$-verifiable in chance for some $\alpha>0$;
\item $H\subseteq W$ is almost surely verifiable;
\item $H\subseteq W$ is open in the weak topology.\\
\end{enumerate}
\end{theorem}

\noindent For an elementary application of the Theorem, consider, in the setting of Example \ref{coinflip}, the sharp hypothesis that the bias of the coin is exactly $.5$. That hypothesis is almost surely refutable, but it is not almost surely verifiable. Since a topological space is determined uniquely by its open sets, Theorem \ref{openisverifiable} implies that the weak topology is the {\em unique} topology that characterizes statistical verifiability under the weak conditions stated in the antecedent of the theorem.  Thus, under those conditions, the weak topology is not merely a convenient formal tool---it is {\em the} topology of statistical information.   \\

\noindent Here is the promised statistical analogue of Theorem \ref{limverififfsigmatwo}.

\begin{theorem}
\label{limverif}
Suppose that $W$ is a set of Borel measures on $\gs=(\Omega, \ct)$, a metrizable sample space with countable, almost surely clopen basis $\ci$. Then the following are equivalent: 
\begin{enumerate}
\item $H\subseteq W$ is limiting verifiable in chance;
\item $H\subseteq W$ is limiting almost surely verifiable;
\item $H\subseteq W$ is a countable union of closed sets in the weak topology.\\
\end{enumerate}
\end{theorem}

\noindent Finally, there is a natural statistical analogue of Theorem \ref{solvability}.

\begin{theorem}
\label{solvableiffsigma2}
Suppose that $W$ is a set of Borel measures on $\gs=(\Omega, \ct)$, a metrizable sample space with countable, almost surely clopen basis $\ci$. Then the following are equivalent: 
\begin{enumerate}
\item $\cq$ is solvable in chance;
\item $\cq$ is almost surely solvable;
\item $\cq$ partitions $W$ into countable unions of closed sets in the weak topology.\footnote{A similar result is proven in \cite[Theorem 2]{dembo1994topological} under  different conditions. Dembo and Peres do not require their methods to be feasible, so Theorem \ref{limverif} does not straightforwardly generalize their result. It is not difficult to reprove Theorem \ref{limverif} without that requirement to obtain a generalization of the result in \cite{dembo1994topological}.}
\end{enumerate}
\end{theorem}

\section{Related Work}
\label{relatedwork}
\noindent Section \ref{deductiveverif} recapitulates foundational results in topological learning theory. Results stated in that section appear previously in \cite{debrecht2009}, \cite{genin2015choice}, and \cite{baltag2015topology}. The theorems stated in section \ref{results} are new, as far as we can tell. In statistical terminology, our Theorem \ref{openisverifiable} provides necessary and sufficient conditions for the existence of a Chernoff consistent test. Although there is extensive statistical work on pointwise consistent hypothesis testing, we are unaware of any topological result analogous to Theorem \ref{openisverifiable}. The closest work is \cite{ryabko2011learnability}, where a topological characterization is given for consistent hypothesis testing of ergodic processes with samples from a discrete, finite alphabet. That result is incomparable with our own, because, although our work is done in the i.i.d setting, we allow samples to take values in an arbitrary, separable metric space. Furthermore, the topology employed in \cite{ryabko2011learnability} is not the weak topology, but the topology of distributional distance. The existence of uniformly consistent tests is investigated topologically in \cite{ermakov2013distinguishability}, where some sufficient conditions are given. Limiting statistical solvability, or {\em discernability}, as it is known in the statistical literature, has been investigated topologically in \cite{dembo1994topological} and \cite{kulkarni1995general}. The results of \cite{dembo1994topological} are generalized to ergodic processes in \cite{nobel2006hypothesis}. Although the setting is slightly different, our Theorem \ref{solvableiffsigma2} gives a simpler back-and-forth condition than the one given in \cite{dembo1994topological} and is arrived at more systematically, by building on the fundamental Theorem \ref{openisverifiable}. The weak topology is used in \cite{dembo1994topological}, but our Theorem \ref{openisverifiable} shows that the weak topology is the {\em unique} topology for which the open sets are exactly the statistically verifiable propositions. Our result shows, therefore, that the weak topology is more than just a convenient technical device.

\section{Conclusion}
\noindent This note lifts the topological perspective on empirical inquiry to statistics. In the deductive setting, open sets are deductively verifiable by true, propositional information. Theorem \ref{openisverifiable} exhibits a topology on probability measures in which the open sets are exactly the propositions statistically verifiable from random samples. In the deductive setting, learnability in the limit receives an elegant topological characterization \cite{baltag2015topology,genin2015choice}. Theorems \ref{limverif} and \ref{solvableiffsigma2} provide analogous topological characterizations of learnability in the limit from statistical data. In light of those fundamental bridge results, we expect many of the streamlined insights of formal learning theory to apply literally to the concrete statistical problems that arise in statistics and machine learning. Of particular interest is the learning theoretic vindication of Ockham's razor, developed topologically in \cite{genin2015choice}, and \cite{kellygeninSL2016}.

%
\bibliographystyle{eptcs}
\bibliography{statock}  
%
%
\appendix
\section{Proofs and Lemmas}
\subsection{Deductive Verifiability}\noindent\begin{proof}[Proof of Theorem \ref{verifiableiffopen}]
{\em Right to left.} Suppose that $H$ is open, and that $H$ is true in $w$. Let $L(E)=H$ if $E$ entails $H$, and let $L(E)=W$ otherwise. Since $H$ is a union of information states, there is an information state $F$ true in $w$ that entails $H$. Therefore,  $L(F)=H$. Furthermore, for any information state $G$ true in $w$, we have that $L(G\cap F)=H$. So $L$ converges to true belief in $H$. Furthermore, if $E\in\ci_w$ then either $w\in E \subseteq H = L(E)$, or $w\in W = L(E)$, so $L$ avoids error in all worlds.
{\em Left to right.} Suppose that $H$ is not open. Then $H$ is true in some $w$, such that for all information $E$ true in $w$, $E$ does not entail $H$, i.e. there is $w\in H\cap \bdry(H)$. Suppose, for contradiction, that $L$ verifies $H$. Then $L(F)\subseteq H$, for some $F$ true in $w$. But, by assumption, there is $v\in F\cap H^{\sf c}$. So $L$ does not avoid error in $v$.  
\end{proof}

\begin{proof}[Proof of Theorem \ref{limverififfsigmatwo}]
{\em Left to right}. Suppose that $L(\cdot)$ is a limiting verifier of $H$. Let $${\cal T} = \{ E \in \ci : L(E) \subseteq H \}.$$
For each $E \in {\cal T}$, let ${\cal D}_E = \{ F \in \ci : F\subseteq E \text{ and } L(F) \nsubseteq L(E)\}$, and let $E' = \bigcup {\cal D}_E$. We claim that: $$H = \bigcup_{E\in{\cal T}} E\setminus E'.$$ To prove the claim, $w\in H$ iff there is $E\in \ci_w$ such that for all information states $F\subseteq E$, $L(F)\subseteq L(E)\subseteq H$ iff there is $E\in {\cal T}$ such that $w\in E\setminus E'$. Since $\ct \subseteq \ci$, and $\ci$ is countable, $H$ is expressed as a countable union of locally closed sets. If the topology is metrizable, every open set --- and therefore every locally closed set --- can be expressed as a countable union of closed sets. {\em Right to left.} Every countable union of locally closed sets can be expressed as a disjoint union of locally closed sets \cite[Proposition 3]{baltag2015topology}. Let $H = \sqcup_{i=1}^\infty O_i \setminus O_i'$ be a disjoint union, for $O_i, O_i'$ open. Let $\sigma(E)$ be the least $i$ such that $E\subseteq O_i$ and $E\nsubseteq O_i'$, if such an $i$ exists, and let $\sigma(E) = \omega$ otherwise. Let $L(E) = O_{\sigma(E)}\setminus O_{\sigma(E)}'$ if $\sigma(E) < \omega$, and let $L(E) = W$ otherwise. Suppose that $w\in H$. Let $k$ be the least integer such that $w\in O_k\setminus O_k'$. Then for $j<k$, either $w\notin O_j$ or $w\in O_j'$. For each $j<k$, let $E_j$ be an information state true in $w$ such that $E_j \subseteq O_j'$, if $w\in O_j'$, and let $E_j = W$, otherwise. Let $E_k$ be an information state true in $w$ that entails $O_k$. Finally, let $E=\bigcap_{j\leq k} E_k$. Then $L(F) = O_k \setminus O_k'$, for every $F\in \ci_w$ such that $F\subseteq E$. Finally, suppose that $w\notin H$. Then for each $j$, either $w\notin O_j$ or $w\in O_j'$. Suppose that $L(E)=O_i \setminus O_i'$, for some $i$. Then $w\in O_i'$. Let $F$ be an information state true in $w$ and entailing $O_i'$. Then $L(E\cap F) \nsubseteq L(E)$, because $L(E \cap F)$ is either $W \not\subseteq O_{i}\setminus O'_{i}$, or $L(E \cap F)$ is some $O_{j}\setminus O'_{j}$, which was chosen to be disjoint from $O_{i} \setminus O'_{i}$. 
\end{proof}

\begin{proof}[Proof of Theorem \ref{solvability}]
See the proof of Theorem 2 in \cite{genin2015choice}.
\end{proof}

\subsection{The Statistical Setting}
\subsubsection{The Sample Space}

The following Lemma states that is always feasible to perform logical operations (e.g. $\wedge$, $\vee$, and $\neg$) on feasible tests.

\begin{lemma}[Lemma 6.4 \cite{parthasarathy1967probability}]
\label{continuityalgebra}
The almost surely clopen sets in $\mu$, denoted $\cc(\mu)$, form an algebra. 
\end{lemma}

\begin{proof}[Proof of Lemma \ref{continuityalgebra}]
One has that $\Omega \in \cc(\mu)$, since $\bdry(\Omega)=\varnothing$. Moreover, $\cc(\mu)$ is closed under complement, since $\bdry(A) = \bdry(A^{\sf c})$. Furthermore, since $\bdry(A\cup B) \subseteq \bdry(A) \cup \bdry(B)$, it follows that if $A,B\in \cc(\mu)$, then $\mu(\bdry(A\cup B))\leq \mu(\bdry(A) \cup \bdry (B)) \leq \mu(\bdry (A)) + \mu(\bdry (B)) = 0$. Therefore, $\cc(\mu)$ is closed under finite union as well.
\end{proof}
\noindent Hypothesis tests are often constructed to reject if the number of samples landing in a particular region exceeds some threshold. The following lemma states that such a test is $\mu$-feasible, if the region is almost surely clopen in $\mu$.

\begin{lemma}
\label{feasibletest}
Suppose that $A$ is almost surely clopen in $\mu$. Then: 
\[ \left\{(\omega_1, \ldots, \omega_n) : \sum_{i=1}^n \mathbb{1}[\omega_i \in A] \geq k\right\} \] 
is almost surely clopen in $\mu^n$, for $n\geq 1$, and $k\in \{0, \ldots n\}$.
\end{lemma}
\begin{proof}[Proof of Lemma \ref{feasibletest}]
Let $L_1, L_2, \ldots, L_{n\mathsf{C}k}$ enumerate all $k$-element subsets of $\{1,2, \ldots, n\}$. Then 
\[\{(\omega_1, \ldots, \omega_n) : \sum_{i=1}^n \mathbb{1}[\omega_i \in A] \geq k\} = \bigcup_{i=1}^{n\mathsf{C}k} \times_{j=1}^n B_{ij}, \]
where $B_{ij}=A$ if $j\in L_i$, and $B_{ij} = \Omega$ otherwise. Since the almost surely clopen sets in $\mu^n$ form an algebra, it suffices to show that $\times_{j=1}^n B_{ij}$ is an almost surely clopen set in $\mu^n$. Argue by induction on $n$. If $n=1$, then $\times_{j=1}^n B_{ij}$ is either $A$ or $\Omega$, which are both almost surely clopen sets in $\mu$. For the inductive step, note that $\bdry(\times_{j=1}^{n+1} B_{ij})\subseteq \bdry(\times_{j=1}^{n} B_{ij})\cup \bdry(B_{i,n+1})$. By the induction hypothesis, $\mu^{n+1}(\bdry(\times_{j=1}^{n+1} B_{ij})) \leq  \mu^n(\bdry(\times_{j=1}^{n} B_{ij}))+\mu(\bdry(B_{i,n+1}))= 0$.
\end{proof}
\subsubsection{The Weak Topology}
\noindent Billingsley \cite{billingsley} proves the following result about the product space:

\begin{lemma}[Theorem 2.8]
\label{product}
If $\ct = \ct' \times \ct''$ is metrizable and second-countable, then $\mu_n'\times\mu_n'' \Rightarrow \mu' \times \mu''$ iff $\mu_n' \Rightarrow \mu'$ and $\mu_n'' \Rightarrow \mu''$.
\end{lemma}

\noindent Lemma \ref{product} entails that the product map $\mu\mapsto\mu^n$ is sequentially continuous, and therefore, continuous, for all natural $n$.\footnote{Continuity is relative to the weak topologies on $W$ and $W^n(\gs) = \{\mu^n : \mu \in W\}$. Recall that a function $f$ is sequentially continuous if whenever a sequence $(x_n)$ converges to a limit $x$, the sequence $f(x_n)$ converges to $f(x)$. In first-countable spaces, sequential continuity is equivalent to continuity. } Billingsley \cite{billingsley} also proves the following useful lemma:

\begin{lemma}[Theorem 2.2]
\label{separatingclass}
Suppose that $\ca\subseteq\cb$ is a $\pi$-system\footnote{$\ca$ is a $\pi$-system iff $A\cap B \in \ca$ whenever $A,B\in\ca$.}  and that every open set is a countable union of $\ca$ sets. If $\mu_n(A) \rightarrow \mu(A)$ for every $A$ in $\ca$, then $\mu_n \Rightarrow \mu$.
\end{lemma}

\noindent The following is a consequence of Lemma \ref{separatingclass}.
\begin{lemma}
\label{countablebasis2}
Suppose that $\ci$ is a countable, almost surely clopen basis for $W$. Then the collection  $\{ \mu : \mu(A) > b \}$ for $A\in\ca$ and $b\in \mathbb{Q}$ is a countable sub-basis for the weak topology.
\end{lemma}
\begin{proof}[Proof of Lemma \ref{countablebasis2}]
It is sufficient to show that $\mu_n \Rightarrow \mu$ iff the $\mu_n$ converge to $\mu$ in the topology generated by the sub-basis. {\em Left to right.} Suppose $\mu_n \Rightarrow \mu$. Let $E$ be open in the topology generated by the sub-basis. Suppose $\mu$ lies in $E$. Then there is a basic open set: \[B=\bigcap_{i=1}^k \{ \mu : \mu(A_i) > b_i \}, \] such that $\mu\in B\subseteq E$. Since $\ci$ is feasible for $W$, $\mu_n(B_i) \rightarrow \mu(B_i)$ for each $i$. Therefore, there exists $n_i$ such that $\mu_n \in \{\mu : \mu(A_i) > b_i\}$ for all $n\geq n_i$. Letting $m=\max\{n_1, \ldots, n_k\}$, it follows that $\mu_n \in B\subseteq E$ for all $n\geq m$. Therefore, the $\mu_n$ converge to $\mu$ in the topology generated by the sub-basis. {\em Right to left.} Suppose that the $\mu_n$ converge to $\mu$ in the topology generated by the sub-basis. Note that since $\ca$ is an algebra, the collection $\{\mu: \mu(A) \in (a,b)\}$ for $A \in \ca$, and $a,b \in \mathbb{Q}$ generates the same topology.\footnote{Notice that $\{\mu: \mu(A) \in (a,b)\}=\{\mu: \mu(A) > a\}\cap \{ \mu : \mu(A^{\sf c}) > 1-b\}.$}  Let $A_1, A_2, \ldots$ enumerate the elements of $\ca(\ci)$. Let $l_{ij}<\mu(A_i)< r_{ij}$ be rationals lying in $(\mu(A_i)-1/j,\,\mu(A_i)+1/j)$. Let $Z_{ij}$ denote the sub-basis element $\{ \nu : \nu(A_i) \in (l_{ij},r_{ij}) \}$. Let $f$ be a surjective function from $\mathbb{N}$ to $\mathbb{N}\times \mathbb{N}$. Let $U_k = Z_{f(k)}$. By assumption, for every $m\geq 1$, there is $n_0$ such that the $\mu_n$ lie in $\cap_{k=1}^m U_k$, for all $n\geq n_0$. So $\mu_n(A) \rightarrow \mu(A)$, for every $A\in\ca$. By Lemma \ref{separatingclass}, $\mu_n \Rightarrow \mu$. 
\end{proof}

\begin{proof}[Proof of Lemma \ref{countablebasis}]
Immediate corollary of Lemma \ref{countablebasis2}.
\end{proof}
\subsection{Statistical Verifiability}
\noindent First, a useful lemma.
\begin{lemma}
\label{asclosedunderconjdisj}
The almost surely verifiable propositions are closed under finite conjunctions, and countable disjunctions.
\end{lemma}
\begin{proof}[Proof of Lemma \ref{asclosedunderconjdisj}]
Suppose that $A_1, A_2$ are a.s. verifiable. Let $\alpha>0$. Let $\{\lambda_n^i\}_{n\in\mathbb{N}}$ be such that $\{\lambda_n^i\}_{n\in\mathbb{N}}$ is an a.s. $\alpha$-verifier for $A_i$. Let $\lambda_n(\vec{\omega}) = A_1 \cap A_2$ if $\lambda_n^i(\vec{\omega}) = A_i$, for $i\in\{1,2\}$. By Lemma \ref{continuityalgebra}, $\lambda_n$ is feasible, for each $\mu\in W,\, n\in\mathbb{N}$. Suppose that $\mu \in A_1 \cap A_2$. Then: 

\begin{align*}
&\mu^\infty\left[\liminf_{n\rightarrow \infty} \lambda_n^{-1}(A_1\cap A_2)\right] = \\
&= \mu^\infty\left[\liminf_{n\rightarrow \infty} (\lambda_n^1)^{-1}(A_1) \cap (\lambda_n^2)^{-1}(A_2)\right] \\
&= 1 - \mu^\infty\left[\limsup_{n\rightarrow \infty} (\lambda_n^1)^{-1}(W) \cup (\lambda_n^2)^{-1}(W)\right]  \\
&=1 -  \mu^\infty\left[\limsup_{n\rightarrow \infty} (\lambda_n^1)^{-1}(W) \cup \limsup_{n\rightarrow \infty}(\lambda_n^2)^{-1}(W)\right]\\
&\geq 1 -  \mu^\infty\left[\limsup_{n\rightarrow \infty} (\lambda_n^1)^{-1}(W)\right]- \mu^\infty\left[\limsup_{n\rightarrow \infty}(\lambda_n^2)^{-1}(W)\right]\\  
&= -1 + \mu^\infty\left[\liminf_{n\rightarrow \infty} (\lambda_n^1)^{-1}(A_1)\right]+ \mu^\infty\left[\liminf_{n\rightarrow \infty}(\lambda_n^2)^{-1}(A_2)\right]\\
&=1.
\end{align*}
 
\noindent Suppose that $\mu \notin A_1 \cap A_2$. Without loss of generality, suppose $\mu \notin A_1$. Then: 

\begin{align*}
&\sum_{n=1}^\infty \mu^\infty\left[\lambda_n^{-1}(A_1 \cap A_2)\right] =\\
&= \sum_{n=1}^\infty\mu^\infty\left[(\lambda_n^1)^{-1}(A_1)\cap (\lambda_n^2)^{-1}(A_2) \right]\\
&\leq \sum_{n=1}^\infty \mu^\infty\left[(\lambda_n^1)^{-1}(A_1)\right]\leq \alpha.
\end{align*}

\noindent To show that the a.s. verifiable propositions are closed under countable union, suppose that $A_1, A_2, \ldots$ are a.s. verifiable. For $i \in \mathbb{N}$, let $\{\lambda_n^i\}_{n\in \mathbb{N}}$ be an a.s. $\alpha_i$-verifier for $A_i$ with $\alpha_i=\alpha/{2^i}$. Let $\lambda_n(\vec{\omega})= \bigcup_{i=1}^\infty A_i$ if $\lambda_n^i(\vec{\omega}) = A_i$ for some $i\in\{1,\ldots, n\}$, and let $\lambda_n(\vec{\omega})=W$ otherwise. By Lemma \ref{continuityalgebra}, $\lambda_n$ is feasible for each $\mu\in W, n\in\mathbb{N}$. Suppose that $\mu \in \bigcup_{i=1}^\infty A_i$. Then there exists $j\in\mathbb{N}$ such that $\mu \in A_j$. Furthermore:

\begin{align*}
\mu^\infty\left[\liminf_{n\rightarrow\infty} \lambda_n^{-1}(\cup_{i=1}^\infty A_i)\right]&= \mu^\infty\left[\liminf_{n\rightarrow\infty} \cup_{k\leq n}(\lambda_n^k)^{-1}(A_k)\right]\\
&\geq \mu^\infty\left[\liminf_{n\rightarrow\infty} (\lambda_n^j)^{-1}(A_j) \right]=1.
\end{align*}

\noindent Suppose that $\mu\notin \cup_{i=1}^\infty A_i$. Then:

\begin{align*}
\sum_{n=1}^\infty \mu^\infty\left[\lambda_n^{-1}(\cup_{i=1}^\infty A_i)\right]&=\sum_{n=1}^\infty \mu^\infty\left[\cup_{k=1}^n (\lambda_n^k)^{-1}(A_k) \right]\\
&\leq \sum_{n=1}^\infty \mu^\infty\left[\cup_{k=1}^\infty(\lambda_n^k)^{-1}(A_k) \right]\\
&\leq \sum_{k=1}^\infty \sum_{n=1}^\infty \mu^{\infty}\left[(\lambda_n^k)^{-1}(A_k) \right]\\
&\leq \sum_{k=1}^\infty \alpha/2^k = \alpha.
\end{align*}
\end{proof}

\begin{proof}[Proof of Theorem \ref{openisverifiable}]
1 implies 3. Suppose, for contradiction, that $H$ is not open, but that $\{\lambda_n\}_{n\in\mathbb{N}}$ is an $\alpha$-verifier in chance for $H$. Let $\mu\in H \cap \bdry H$.  Then there is a sequence of $\mu_n$ in $H^{\sf c}$ such that $\mu_n \Rightarrow \mu$. Since $\{\lambda_n\}_{n\in\mathbb{N}}$ is a verifier for $H$, there is a sample size $k$ such that $\mu^k(\lambda_k^{-1}(H)) > \alpha + \epsilon$. By Lemma \ref{product}, $\mu_n^k(\lambda_k^{-1}(H)) \rightarrow \mu^k(\lambda_k^{-1}(H))$. So there is a $\mu_m\in H^{\sf c}$ such that $\mu_m^k(\lambda_k^{-1}(H))>\alpha$. Contradiction.\\\\
3 implies 2.  By Lemmas \ref{countablebasis}, \ref{asclosedunderconjdisj}, it is sufficient to show that every element of the sub-basis: \[\left\{ \{ \mu : \mu(A) > b \} : A \in\ca(\ci), b\in \mathbb{Q} \right\}\] is a.s. verifiable. Let $B\in\ca(\ci)$, and let $H=\{ \mu : \mu(B) > b\}$, for some $b\in[0,1]\cap\mathbb{Q}$. Define the indicator random variable $\mathbb{1}_B:\Omega \rightarrow \{0,1\}$ by $\mathbb{1}_B(\omega)=1$ if $\omega\in B$, otherwise  $\mathbb{1}_B =0$.  Letting $t_n=\sqrt{\frac{1}{2n}\ln(\pi^2 n^2/6\alpha)}$, it follows from Hoeffding's inequality that:
\[ \mu^n\left[\sum_{i=1}^n \mathbb{1}_B(\omega_i) \geq n\left(\mu(B) + t_n\right)\right] \leq \frac{6\alpha}{\pi^2 n^2}. \]

\noindent Let $\lambda_n(\vec{\omega}) = H$ if $\sum_{i=1}^n \mathbb{1}_B(\omega_i) \geq \lceil{n(b+t_n)\rceil}$, and let $\lambda_n(\vec{\omega}) = W$ otherwise. By Lemmas \ref{feasibletest}, $\lambda^n$ is feasible for all $\mu^n$.  If $\mu\notin H$, then $b\geq\mu(B)$ and: 
\begin{align*}
\sum_{n=1}^\infty\mu^n\left[\lambda_n^{-1}(H)\right] &= \sum_{n=1}^\infty \mu^n\left[ \sum_{i=1}^n \mathbb{1}_B(\omega_i) \geq \lceil{n(b+t_n)\rceil}\right] \\
&\leq \sum_{n=1}^\infty \mu^n\left[ \sum_{i=1}^n \mathbb{1}_B(\omega_i) \geq n(\mu(B)+t_n)\right]\\
&\leq \sum_{n=1}^\infty \frac{6\alpha}{\pi^2 n^2} = \alpha.
\end{align*}
\noindent Furthermore, if $\mu\in H$, then since $\frac{1}{n} \sum_{i=1}^n \mathbb{1}_B(\omega_i)\overset{a.s.}{\rightarrow} \mathbb{E}[\mathbb{1}_B] = \mu(B)$, by the strong law of large numbers, and $t_n\rightarrow 0$, we have that $\frac{1}{n} \sum_{i=1}^n \mathbb{1}_B(\omega_i) - t_n \overset{a.s.}{\rightarrow} \mu(B)$. Therefore, $\mu^\infty\left[\underset{n\rightarrow \infty}{\liminf} \hspace{1pt} \lambda_n^{-1}(H)\right]=1$, as required. \\\\
\noindent 2 implies 1. Immediate from the definitions.
\end{proof}

\begin{proof}[Proof of Theorem \ref{limverif}]
1 entails 3. Suppose that $\{\lambda_n\}_{n\in\mathbb{N}}$ is a limiting verifier in chance of $H$. For every $H'\in {\sf rng}(\lambda_n)$, let:
\begin{align*}
{\sf trig}_n(H') &= \{\mu : \mu^n[\lambda_n^{-1}(H')] > \alpha \};\\
{\sf def}_n(H') &= \cup_{m> n} \{ \mu : \mu^n[\lambda_n^{-1}(H')^{\sf c}] > 1 - \alpha\}.
\end{align*}
\noindent Lemma \ref{product} and the feasibility of the $\lambda_n$ entail that ${\sf trig}_n(H')$ and ${\sf def}_n(H')$ are both open in the weak topology. Let $\alpha\in(0,1)$. We claim that: 
\begin{align*}
H &= \bigcup_{n=1}^\infty \bigcup_{H'\in {\sf rng}(\lambda_n)} {\sf trig}_n(H') \setminus {\sf def}_n(H').
\end{align*}
Observe that $\nu \in H$ iff there is $H'\subseteq H \text{ and } m\in\mathbb{N}$, such that for all $n\geq m, \text{ } \nu^n[\lambda_n^{-1}(H')]> \alpha$ iff $\nu\in {\sf trig}_n(H')\setminus {\sf def}_n(H').$ Therefore, proposition $H$ is a countable union of locally closed sets. Since the weak topology is metrizable, every open set --- and therefore every locally closed set --- can be expressed as a countable union of closed sets.\\

\noindent 3 entails 2. Suppose that $H = \cup_{i=1}^\infty C_i$ is a countable union of closed sets. By theorem \ref{openisverifiable}, for each $C_i$, there exists an a.s. statistical $\alpha$-verifier of its complement $\{\psi_n^i\}_{n\in\mathbb{N}}$. Let $\lambda_n(\vec{\omega}) = C_j$, where $j$ is the least integer in $1, \ldots, n$ such that $\psi^j_n(\vec{\omega}) = W$, if such a $j$ exists. Otherwise, let $\lambda_n(\vec{\omega}) = W$. Suppose that $\mu \in H$, and that $k$ is the least integer such that $\mu\in C_k \subseteq H$. Then:
\begin{align*}
&\mu^\infty\left[\liminf_{n\rightarrow\infty}\lambda_n^{-1}(C_k)\right] =\\
&= \mu^\infty\left[\liminf_{n\rightarrow \infty} \bigcap_{j<k} (\psi^j_n)^{-1}(C_j^{\sf c}) \cap (\psi^k_n)^{-1}(W)\right]\\
&= 1 - \mu^\infty\left[\limsup_{n\rightarrow\infty} \bigcup_{j<k} (\psi^j_n)^{-1}(W) \cup (\psi^k_n)^{-1}(C_k^{\sf c})\right]\\
&=1 - \mu^\infty\left[\bigcup_{j<k}\limsup_{n\rightarrow\infty}\hspace{1pt} (\psi^j_n)^{-1}(W) \cup \limsup_{n\rightarrow\infty} \hspace{1pt}(\psi^k_n)^{-1}(C_k^{\sf c})\right]\\ 
&\geq 1 - \sum_{j<k}\mu^\infty\left[\limsup_{n\rightarrow \infty} \hspace{1pt} (\psi_n^j)^{-1}(W)\right] - \mu^\infty\left[\limsup_{n\rightarrow \infty} \hspace{1pt} (\psi^k_n)^{-1}(C_k^{\sf c})\right].
\end{align*}

\noindent Since each $\{\psi_n^j\}_{n\in\mathbb{N}}$ is an a.s. $\alpha$-verifier of $C_j$, and $\mu \notin \cup_{j<k} C_j$, it follows that\\ $\mu^{\infty}\left[\limsup_{n\rightarrow \infty} (\psi_n^j)^{-1}(W)\right]=0$ for $j<k$. Therefore, 
\begin{align*}
\mu^\infty\left[\liminf_{n\rightarrow\infty}\lambda_n^{-1}(C_k)\right] &\geq 1 - \mu^\infty\left[\limsup_{n\rightarrow \infty} \hspace{1pt} (\psi^k_n)^{-1}(C_k^{\sf c})\right].
\end{align*}
\noindent Finally, since $\sum_{n=1}^\infty \mu^\infty\left[(\psi_n^k)^{-1}(C_k^{\sf c}) \right]\leq \alpha$,  by the Borel-Cantelli lemma, $$\mu^\infty\left[\limsup_{n\rightarrow \infty} \hspace{1pt} (\psi^k_n)^{-1}(C_k^{\sf c})\right]=0,$$
\noindent and $\mu^\infty\left[\liminf_{n\rightarrow\infty}\lambda_n^{-1}(C_k)\right]=1$, as required. Suppose that $\mu \notin H$. Then, for each $C_k$, $$\mu^\infty[\limsup_{n\rightarrow \infty} \lambda_n^{-1}(C_k)]\leq \mu^\infty[\limsup_{n\rightarrow \infty} (\psi_n^k)^{-1}(W)]=0.$$

\noindent 2 entails 1. Immediate from the definitions.
\end{proof}
\begin{proof}[Proof of Theorem \ref{solvableiffsigma2}]
1 entails 3. Suppose that the indexed set $\{\lambda_n\}_{n\in\mathbb{N}}$ is a solution in chance to $\cq$, and that $A$ is an answer to $\cq$. Then, $\{\lambda_n\}_{n\in\mathbb{N}}$ is a limiting verifier in chance of $A$. Therefore, by Theorem \ref{limverif}, $A$ is a countable union of closed sets in the weak topology. \\

\noindent 3 entails 2. Let $A_1, A_2, \ldots$ enumerate the answers to $\cq$. Suppose that each answer is a countable union of closed sets, i.e. that each $A_j$ in $\cq$ can be expressed as a countable union of closed sets $\cup_{i=1}^\infty C_{ij}$. By Theorem \ref{openisverifiable}, for each $C_{ij}$, there exists an almost sure $\alpha$-verifier $\{\psi^{ij}_n\}_{n\in\mathbb{N}}$ of its complement. Let $f$ be a surjective function from $\mathbb{N}$ to $\mathbb{N}\times\mathbb{N}$. Let $\lambda_n(\vec{\omega})=\cq(C_{f(k)})$, where $k$ is the least natural number in $\{1,\ldots, n\}$ such that $\psi_n^{f(k)}(\vec{\omega})=W$, if such a $k$ exists, and let $\lambda_n(\vec{\omega})=W$, otherwise. It is easy to see that each $\lambda_n$ is feasible by Lemma \ref{continuityalgebra}. Suppose that $\mu\in W$. Claim: $\mu^\infty[\liminf_{n\rightarrow\infty} \lambda_n^{-1}(\cq_\mu)]=1$. To establish the claim, let $k$ be the least natural number such that $\mu\in C_{f(k)}$. Then:
\begin{align*}
&\mu^\infty\left[\liminf_{n\rightarrow\infty} \lambda_n^{-1}(\cq_\mu)\right] \geq\\
&\geq \mu^\infty\left[\liminf_{n\rightarrow\infty} \hspace{1pt} \bigcap_{i<k} \left(\psi^{f(i)}_n\right)^{-1} \left(C_{f(i)}^{\sf c}\right) \cap \left(\psi_n^{f(k)}\right)^{-1}(W)\right]\\
&= 1 - \mu^\infty\left[\limsup_{n\rightarrow\infty} \hspace{1pt} \bigcup_{i<k} \left(\psi^{f(i)}_n\right)^{-1} \left(W\right) \cup \left(\psi_n^{f(k)}\right)^{-1}\left(C_{f(k)}^{\sf c}\right)\right]\\
&= 1 - \mu^\infty\left[ \bigcup_{i<k} \limsup_{n\rightarrow\infty} \left(\psi^{f(i)}_n\right)^{-1} \left(W\right) \cup \limsup_{n\rightarrow\infty} \left(\psi_n^{f(k)}\right)^{-1}\left(C_{f(k)}^{\sf c}\right)  \right]\\
&\geq 1 - \sum_{i<k} \mu^{\infty}\left[ \limsup_{n\rightarrow\infty} \left(\psi^{f(i)}_n\right)^{-1} \left(W\right)  \right]- \mu^\infty\left[ \limsup_{n\rightarrow\infty} \left(\psi_n^{f(k)}\right)^{-1}\left(C_{f(k)}^{\sf c}\right) \right].
\end{align*}
\noindent Since each $\{\psi_n^{f(i)}\}_{n\in\mathbb{N}}$ is an a.s. $\alpha$-verifier of $C_{f(i)}$, and $\mu \notin \cup_{i<k} C_{f(i)}$, it follows that\\ $\mu^{\infty}\left[\limsup_{n\rightarrow \infty} (\psi_n^{f(i)})^{-1}(W)\right]=0$ for $i<k$. Therefore, 

\[ \mu^\infty\left[\liminf_{n\rightarrow\infty} \lambda_n^{-1}(\cq_\mu)\right] \geq 1 -  \mu^\infty\left[ \limsup_{n\rightarrow\infty} \left(\psi_n^{f(k)}\right)^{-1}\left(C_{f(k)}^{\sf c}\right) \right]. \]

\noindent Finally, since $\sum_{n=1}^\infty \mu^\infty\left[\left(\psi_n^{f(k)}\right)^{-1}\left(C_{f(k)}^{\sf c}\right) \right]\leq \alpha$, the Borel-Cantelli lemma yields: $$\mu^\infty\left[\limsup_{n\rightarrow \infty} \hspace{1pt} \left(\psi^{f(k)}_n\right)^{-1}\left(C_{f(k)}^{\sf c}\right)\right]=0,$$
\noindent so $\mu^\infty\left[\liminf_{n\rightarrow\infty}\lambda_n^{-1}(\cq_\mu)\right]=1$, as required.
      \\

\noindent 2 entails 1. Immediate from the definitions.
\end{proof}

\end{document}